%% file: main.tex
\documentclass{article}

\usepackage[final]{corl_2026} 

\usepackage[utf8]{inputenc}
\usepackage[T1]{fontenc}
\usepackage{amsmath, amssymb, amsfonts, amsthm, mathtools}
\usepackage{bm}
\usepackage{booktabs}
\usepackage[table,xcdraw]{xcolor}
\usepackage[most]{tcolorbox}
\usepackage{graphicx}
\usepackage{wrapfig}
\usepackage{float}
\usepackage{enumitem}
\usepackage{makecell}

\newcommand{\SO}{\mathrm{SO}}

\newcommand{\Sph}{\mathbb{S}^{2}}
\newcommand{\RR}{\mathbb{R}}
\newcommand{\OO}{\mathcal{O}}
\newcommand{\LL}{\mathcal{L}}
\newcommand{\eps}{\epsilon}

\newtheorem{proposition}{Proposition}

\title{Equivariant Visual-Tactile Diffusion Policy for Contact-Rich Manipulation}

\author{
\textbf{WONG Lik Hang Kenny\textsuperscript{1},
Yiyao Ma\textsuperscript{1},
Xiu-Shen Wei\textsuperscript{2},
Zelong Tan\textsuperscript{1},}\\
\textbf{Zhuheng Song\textsuperscript{1},
Dongsheng Xie\textsuperscript{1},
Kai Chen\textsuperscript{1},
Qi Dou\textsuperscript{1}}\\[1mm]
\textsuperscript{1}Department of Computer Science and Engineering, The Chinese University of Hong Kong\\
\textsuperscript{2}School of Computer Science and Engineering, Southeast University
}

\begin{document}
    \maketitle


\begin{abstract}
	Imitation learning for contact-rich manipulation requires high-quality expert data that is expensive to obtain. This makes learning a sample-efficient policy a key issue. To address this, we propose VISTA, a workspace-level equivariant visuotactile diffusion policy for data-efficient contact-rich imitation learning. VISTA projects visual and tactile observations into spherical tokens, injects tactile contact cues into visual spherical directions through permutation-equivariant spherical fusion, and rotates the fused harmonic representation using the end-effector orientation. The resulting representation conditions an equivariant diffusion policy to predict spatially consistent actions. Extensive experiments in both simulation and real-world robotic settings show that VISTA substantially improves data efficiency over strong visuotactile imitation learning baselines. Project website: \href{https://vista-paper.github.io/}{https://vista-paper.github.io/}
	\end{abstract}
	
	
	
	\input{sections/introduction}
	
	\input{sections/related_work}
	
	\input{sections/method}
	
	\input{sections/experiment}

	\input{sections/conclusion}
	
	\clearpage
	\input{sections/acknowledgments}
	\bibliography{ref}
	
	\clearpage
	
	\appendix
	\input{sections/appendix}

	
	
	\end{document}

%% file: sections/introduction.tex
\section{Introduction}
\label{sec:introduction}

Many daily contact-rich tasks, such as inserting a USB drive or wiping a table, require tactile sensing in addition to vision. For robots, however, using both modalities effectively remains challenging. Recent robotics research has increasingly adopted imitation learning with visual and tactile feedback for contact-rich manipulation \cite{yu2023mimictouch, george2025vital,xue2025reactive}, but expert demonstrations are expensive to collect and often cover only a limited range of object poses, gripper poses, and contact configurations \cite{tsuji2025survey}. Moreover, for tasks such as insertion, success depends heavily on the policy's robustness to small deviations in observed object and robot poses \cite{kim2022active}. Therefore, how can we learn robust visuotactile policies with fewer demonstrations?

Equivariance provides a principled way to reduce the spatial coverage required from data for a policy to achieve spatial robustness \cite{wangincorporating, cohen2019general, thomas2018tensor, deng2021vector, jiang2019spherical}. In an equivariant policy, if the observation is transformed by a spatial transformation from a symmetry group, the predicted action should transform in the corresponding way.
By building this symmetry into the architecture, the model can improve spatial generalizability across object poses, gripper poses, and scene configurations with fewer demonstrations than non-equivariant policies \cite{wang2022equivariant, wang2022robot}. This is especially appealing for contact-rich manipulation, where demonstrations are costly, and task success is sensitive to small spatial deviations. However, existing equivariant manipulation policies are still largely vision-centric. They are typically designed around RGB images, point clouds, or other global geometric observations \cite{ryu2023equivariant}, limiting their ability to handle contact-rich settings where local tactile feedback is critical.

In this paper, we propose VISTA, a workspace-level equivariant visuotactile diffusion policy. Our method is inspired by the observation that both egocentric camera and tactile sensors are often mounted on the end effector of a robot for manipulation tasks; therefore, global workspace rotations change the sensor-to-world transform while preserving the local contact measurement. Concretely, our method first maps egocentric vision and tactile observations into local spherical representations. It then performs permutation-equivariant spherical fusion, allowing tactile contact cues to modulate visual directions while preserving the spherical structure of the representation. The fused spherical representation is then transformed by the end-effector pose rotation and processed with an equivariant diffusion policy to achieve workspace-level equivariance. As a result, VISTA improves data efficiency and generalization under limited demonstrations. We evaluate VISTA on both contact-rich simulation tasks and real-robot experiments and show that the proposed framework enables more data-efficient contact-rich manipulation learning.

Our main contributions are:
\begin{itemize}[leftmargin=1.2em]
    \item We introduce VISTA, extending the workspace-level equivariant backbone from Image-to-Sphere Policy (ISP) \cite{hu20263d} to the visuotactile setting for contact-rich tasks, enabling more data-efficient learning by incorporating local tactile feedback into an equivariant policy.

    \item We design a spherical cross-modal fusion module that integrates tactile cues into the structured visual representation, allowing the downstream equivariant diffusion policy to leverage the fused representation effectively. Across eight simulated tasks, this design improves mean success by 6.8 percentage points over direct visual-tactile feature concatenation on the sphere.

    \item We demonstrate data-efficient contact-rich manipulation across eight simulated tasks and three real-robot tasks: VISTA achieves 80.9\% mean success, outperforming EquiDiff+Tac by 15.9 percentage points, and 73.3\% average success on real robots.
\end{itemize}

%% file: sections/related_work.tex
\section{Related Work}
\label{sec:related_work}

\subsection{Equivariance in robot manipulation}

Equivariance embeds workspace symmetry into the policy class by requiring that transformed scenes, objects, and robot states yield actions transformed in the same physical manner \cite{seo2025se, cohen2016group, van2020mdp, cohen2016steerable, cohen2019general}. Existing equivariant manipulation methods primarily enforce this constraint in two observation domains. The first operates on 3D geometry such as point clouds \cite{simeonov2022neural,fuchs2020se, thomas2018tensor, deng2021vector,ryu2023equivariant, yang2024equibot, tie2024seed, zhang2025canonical}, where scalar, vector, or higher-order features are attached to 3D points and propagated according to their transformation rules \cite{yang2024equibot,tie2024seed,zhang2025canonical}; translation is typically handled by centering, relative coordinates, centroid recovery, or canonicalization, while rotation can be modeled with vector neurons \cite{deng2021vector}. Although this provides a natural interface for workspace-level equivariance, it often relies on reliable depth sensing, calibration, point cloud extraction, and alignment. The second route starts from RGB observations, motivated by large-scale robot datasets and practical systems based on monocular or wrist-mounted cameras; recent work has explored planar transport, bi-equivariant pick-and-place, equivariant diffusion for 6-DoF actions, and image-to-sphere projection for egocentric RGB inputs~\cite{huang2024fourier,EquiDiff,i2s}. However, both routes remain largely vision-centric: their symmetries are defined over visual observations, proprioception, and actions.

\subsection{Contact-rich visuotactile manipulation}
Tactile sensing provides contact-local cues---geometry, deformation, shear, slip, and force---that vision often cannot resolve, making it crucial for contact-rich manipulation. Existing work studies visuotactile representation learning through contrastive or simulation-based pretraining, while also integrating touch into policies for precise or reactive control~\cite{george2025vital,yang2024binding,chen2026univtac,luu2025manifeel,huang20243d,xue2025reactive, yang2022touch, tian2019manipulation, hogan2020tactile,kim2022active}. Although these methods show the value of tactile feedback when vision is insufficient, they focus mainly on representation, benchmark design, or reactive feedback rather than explicitly modeling workspace-level symmetry. Thus, tactile observations are usually encoded by CNNs, pretrained encoders, deformation features, or temporal modules and then fused with visual and proprioceptive inputs, allowing policies to use touch effectively but without explicitly enforcing consistent transformations among vision, touch, robot state, and 6-DoF actions under changes in object pose, gripper pose, or task frame. VISTA addresses this geometric gap by organizing visual and tactile RGB observations within a shared steerable latent interface before equivariant action generation.

\subsection{Tactile equivariance and geometric contact policies}
The most closely related works also study symmetry in tactile or contact-driven manipulation, but they realize equivariance in different settings and through different mechanisms. EquiTac defines tactile equivariance for in-hand object rotation by reconstructing tactile normal maps whose in-plane components transform as $SO(2)$ vectors while out-of-plane components remain invariant, and then uses an equivariant network to predict an angular residual for correcting a base policy~\cite{zhu2025equitac}; however, this symmetry is local to the fingertip/object-in-hand frame and serves as a residual orientation correction, rather than a global workspace-level policy that jointly transforms RGB observations, tactile feedback, proprioception, and 6-DoF action chunks. In contrast, geometric contact policies achieve workspace-level symmetry from a control-theoretic perspective: contact-rich $SE(3)$-equivariant manipulation relies on body-frame error and wrench representations with impedance control~\cite{seo2023contact}, while hierarchical variants further combine global reference-frame estimation with localized compliant policies \cite{seo2025equicontact}. Thus, although these methods provide important forms of tactile or contact equivariance, VISTA complements them by building workspace-level equivariance directly into an end-to-end diffusion policy over visuotactile and proprioceptive observations.

%% file: sections/method.tex
\section{Method}
\label{sec:method}

VISTA is a workspace-level equivariant visuotactile diffusion policy for contact-rich manipulation. Our method is compatible with a now-common robotics setting where both visual and tactile cameras are rigidly attached to the end effector. In this setting, VISTA first projects visual and tactile observations into spheres and fuses them, then rotates the fused representation into the workspace frame using the end-effector orientation. The representation conditions an equivariant diffusion policy that predicts 6-DoF action chunks. The architecture of VISTA is illustrated in Figure \ref{fig:model}.

\begin{figure}[t]
    \centering
    \includegraphics[width=1\linewidth]{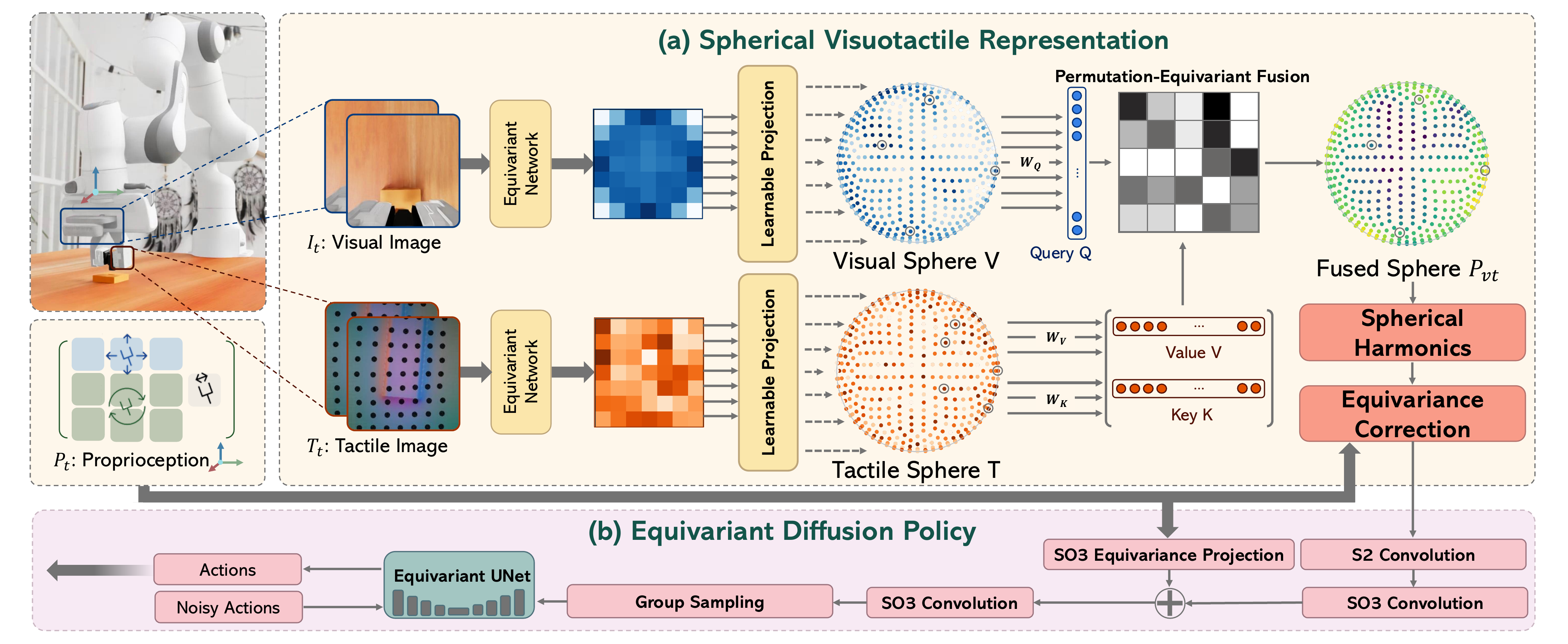}
    \caption{\small
     Model architecture of VISTA. (a) Vision and tactile inputs are encoded into spherical tokens, fused via permutation-equivariant spherical fusion, projected onto spherical harmonics, and aligned with the workspace using the end-effector orientation. (b) The aligned representation conditions a finite-group-equivariant diffusion policy to denoise action chunks and predict 6-DoF actions plus gripper commands.}
    \label{fig:model}
    \vspace{-5mm}
\end{figure}

\subsection{Problem Formulation}
\label{sec:problem_formulation}

We study visuotactile imitation learning for contact-rich manipulation. At time step \(t\), the policy observes a history of observation $\OO_t = \{o_{t-H_o+1},\ldots,o_t\}, o_t=(I_t,T_t,p_t)$, where \(H_o\) is the observation horizon, \(I_t\) is the egocentric vision image, \(T_t\) is the tactile image, and \(p_t\) is the robot proprioceptive state. The policy predicts an action chunk $\mathbf{A}_t = \{a_t,\ldots,a_{t+H_p-1}\} \in \RR^{H_p\times d_a}$, where \(H_p\) is the prediction horizon. Each action is represented as $a_t=(x_t,r^1_t,r^2_t,\gamma_t)\in\RR^{10}$, where \(x_t\in\RR^3\) is the end-effector translation command, \(r^1_t,r^2_t\in\RR^3\) are the first two axes of the 6D rotation representation~\cite{zhou2019continuity}, and \(\gamma_t\) is the gripper command.

Our goal is to learn a workspace-level equivariant visuotactile policy. Let \(G\subseteq\SO(3)\) denote the task-relevant rotation symmetry group. For a transformation \(g\in G\), the policy should satisfy
\begin{equation}
    \pi(g\cdot\OO_t)
    =
    g\cdot\pi(\OO_t).
    \label{eq:equivariance_goal}
\end{equation}

In implementation, \(G\) is instantiated as a finite rotation group such as \(I_{60}\) for tasks requiring 3D rotational structure and \(C_8\) for yaw-dominant tasks. Thus, the equivariance guarantee is exact for the selected finite group and serves as a practical discretization of continuous SO(3) or SO(2) symmetry. Here \(g\cdot\OO_t\) denotes a joint transformation of the robot–object–sensor system, in which the end-effector pose, the rigidly mounted wrist camera and tactile sensor, and the task-relevant object geometry are rotated together. Meanwhile, \(g\cdot\pi(\OO_t)\) denotes the transformation of the action output by the policy.

\subsection{Visuotactile Spherical Representation}
\label{sec:visuotactile_spherical_representation}

Under a global rotation, the relative geometry observed in both the egocentric camera and the tactile sensor is unchanged. Under this joint workspace rotation of the robot-object-sensor system, the ideal sensor-local measurements are invariant: the relative camera-object geometry and local tactile deformation remain unchanged. This property suggests a simple route to workspace equivariance, as proposed in Image-to-Sphere Policy (ISP) \cite{hu20263d}: first encode egocentric vision into a steerable spherical representation, then use the end-effector orientation to rotate the resulting representation into the workspace frame. VISTA additionally lifts tactile observations into spherical tokens and fuses them with the visual field before the same correction.

\paragraph{Spherical projection.} Concretely, for each observation step, VISTA first extracts 2D feature maps from the visual and tactile images using an equivariant ResNet \cite{he2016deep, wangsurprising}. Both the visual and tactile feature maps are then projected to a set of tokens on a discretized sphere \cite{klee2023image, i2s} (details in Sec. ~\ref{app:spherical_lifting}). Concretely, let $V=\{v_i\}_{i=1}^{P}$ denote the visual spherical tokens and $\mathcal{T}=\{\tau_j\}_{j=1}^{P}$ denote the tactile spherical tokens; \(P\) is the number of sampled sphere points for both vision and tactile.

\paragraph{Permutation-Equivariant Spherical Fusion.} VISTA uses a permutation equivariant cross-attention mechanism \cite{vaswani2017attention} to inject tactile evidence into each visual direction. For each visual spherical token \(v_i\), VISTA queries tactile spherical tokens \(\{\tau_j\}_{j=1}^{P}\):
  \begin{align}
      q_i &= W_q v_i,
      &
      k_j &= W_k \tau_j,
      &
      u_j &= W_u \tau_j,
      \label{eq:attention_qku}
      \\
      \alpha_{ij}
      &=
      \operatorname{softmax}_{j}
      \left(
      \frac{q_i^\top k_j}{\sqrt{d_h}}
      \right),
      &
      m_i
      &=
      \sum_j \alpha_{ij}u_j .
      \label{eq:attention_message}
  \end{align}
  Here \(m_i\) is the tactile message aggregated for visual direction \(i\). The fused spherical token is computed by a residual token-wise fusion block:
  \begin{equation}
      f_i
      =
      \operatorname{LN}
      \left(
      v_i
      +
      \psi([v_i,m_i])
      \right),
      \label{eq:fused_token}
  \end{equation}
  where \(\psi\) is a small MLP, and \([v_i,m_i]\) denotes channel-wise concatenation. The resulting tactile-conditioned spherical field is $F=\{f_i\}_{i=1}^{P}$.

\begin{proposition}[Preservation of spherical indexing under tactile fusion]
\label{prop:index_equivariance_fusion}
Let \(\mathcal{A}(V,\mathcal{T})\) denote the fusion module in
Eqs.~\eqref{eq:attention_qku}--\eqref{eq:fused_token}, where
\(V=\{v_i\}_{i=1}^{P}\) and \(\mathcal{T}=\{\tau_j\}_{j=1}^{P}\)
denote the visual and tactile spherical tokens, respectively.
For any permutation matrix \(\Pi\) acting on the visual spherical
tokens,
\[
    \mathcal{A}(\Pi V,\mathcal{T})
    =
    \Pi\mathcal{A}(V,\mathcal{T}).
\]
Hence, tactile conditioning preserves the visual spherical indexing
under the sampled rotations used by the downstream equivariant
representation.
\end{proposition}

\begin{wrapfigure}{r}{0.5\linewidth}
    \centering
    \includegraphics[width=\linewidth]{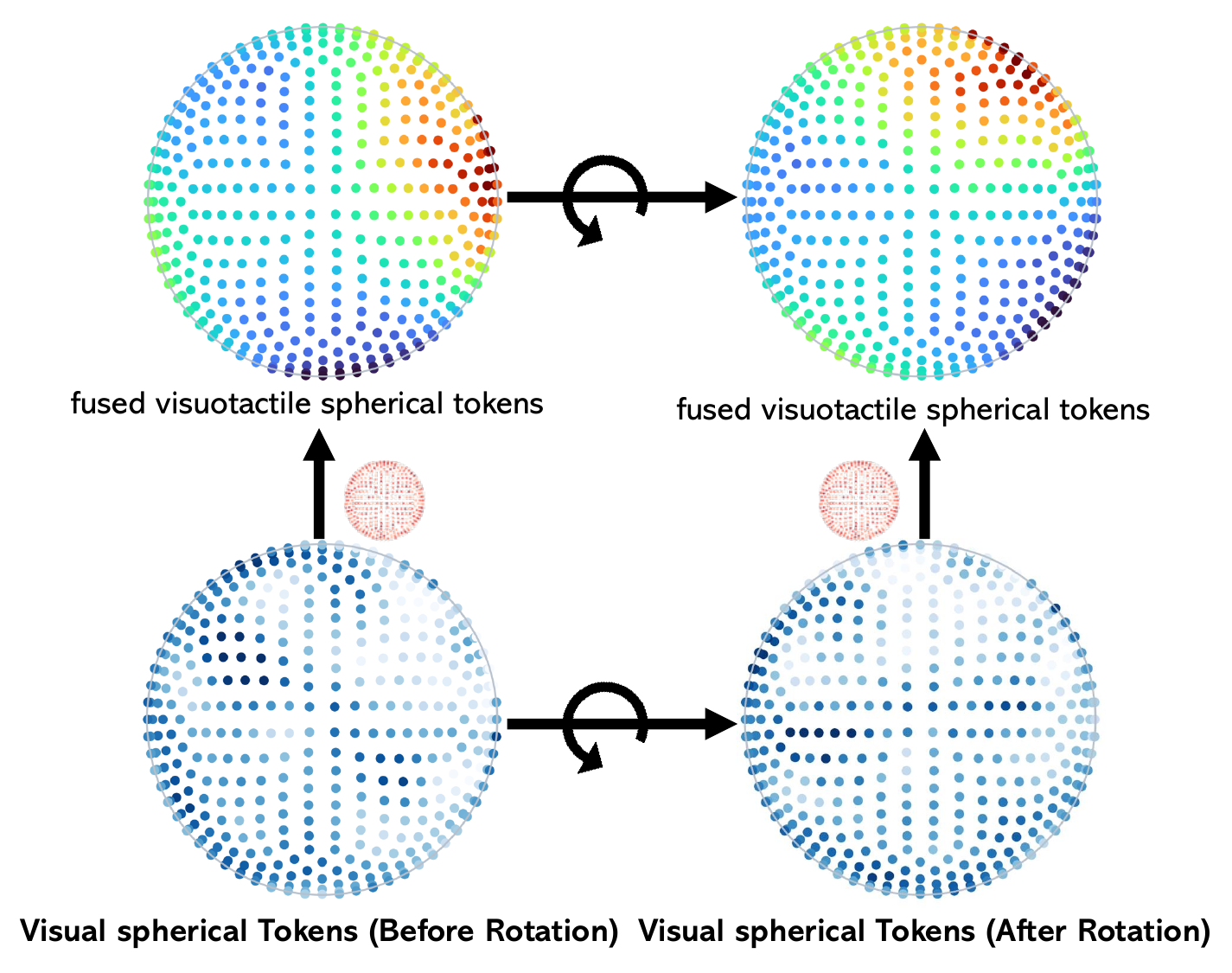}
    \caption{\small \textbf{Permutation-Equivariant Spherical Fusion}. \textbf{Bottom}: visual spherical tokens before and after a token-space rotation, with tactile tokens fixed. \textbf{Top}: fused visuotactile tokens rotate coherently with the visual token structure, showing that tactile cross-attention preserves the spherical indexing.}
    \label{fig:fusion_equivariance}
    \vspace{-5mm}
\end{wrapfigure}

\begin{proof}[Proof sketch]
Permuting the visual tokens only permutes the rows of the query matrix. Since the attention softmax is applied row-wise over tactile tokens, the attention weights and messages are permuted by the same visual permutation. The residual projection, MLP, and layer normalization are applied token-wise, so they also commute with the same permutation. Therefore, the fused output preserves the indexing of the visual sphere. A full proof is provided in Appendix~\ref{app:fusion_proof}.
\end{proof}

Proposition~\ref{prop:index_equivariance_fusion} shows that tactile fusion preserves the directional indexing of the visual sphere. As illustrated in Figure~\ref{fig:fusion_equivariance}, when the visual spherical field is re-indexed by a discretized rotation, the fused field undergoes the same re-indexing. Thus, tactile conditioning can modify the feature content associated with each visual direction without changing its correspondence to the underlying sphere. Therefore, this incorporates tactile information without disrupting the visual geometric structure.

\paragraph{Harmonic projection and equivariance correction.}
Following ISP \cite{hu20263d}, VISTA projects the tactile-conditioned spherical tokens to spherical harmonic coefficients. Let \(s_i\in\Sph\) be the sphere point associated with token \(f_i\), and let \(Y(s_i)\) denote the vector of spherical harmonics up to degree \(l_{\max}\). The fused coefficients are computed as $C_t = \sum_{i=1}^{P} f_i\,Y(s_i)\,w_i$, where \(w_i\) is the quadrature weight. We then align the local sensor representation to the workspace frame using the end-effector orientation \(R_t\). For each harmonic degree \(\ell\), the correction is $\widehat{c}_{\ell,t} = D^\ell(R_t)c_{\ell,t}$, where \(D^\ell\) is the Wigner-\(D\) representation of \(\SO(3)\). The corrected coefficients are processed by spherical convolution layers to produce an \(\SO(3)\)-structured observation feature. Proprioception is encoded using scalar features for gripper state and vector features for end-effector position and orientation axes. The fused visuotactile and proprioceptive features are concatenated to form the final observation representation used by the downstream diffusion policy.

\subsection{Equivariant Diffusion Policy}
\label{sec:equivariant_diffusion_policy}

VISTA uses a diffusion policy \cite{hu20263d,chi2025diffusion} to model multimodal action chunks \cite{ho2020denoising}. Let \(E_\theta(\OO_t)\) denote the encoded visuotactile observation. During training, diffusion noise is applied on a clean action chunk \(\mathbf{A}_t\) at timestep \(k\), producing \(\mathbf{A}_t^k\). The denoising network \(\Gamma_\theta\) is trained to predict the added noise:
\begin{equation}
    \LL_{\mathrm{diff}}
    =
    \mathbb{E}
    \left[
    \left\|
    \Gamma_\theta(\mathbf{A}_t^k,E_\theta(\OO_t),k)
    -
    \eps
    \right\|_2^2
    \right].
    \label{eq:diffusion_loss}
\end{equation}
At inference time, VISTA samples an action chunk by iteratively denoising a Gaussian action sequence conditioned on the current visuotactile observation history. To make the denoising process equivariant, VISTA represents both the noisy action chunk and observation condition as steerable features. The translational action and two rotation axes are treated as vector features, while the gripper command is treated as a scalar. The equivariant decoder is instantiated with a rotation group chosen according to the task symmetry. For tasks requiring full spatial rotational equivariance, we use an \(I_{60}\) icosahedral subgroup to represent the \(SO(3)\) setting. For tasks where the relevant symmetry is rotation around the gravity axis, we use a \(C_8\) cyclic group to represent the \(SO(2)\) setting.

\subsection{Equivariance Analysis}
\label{sec:equivariance_analysis}

\paragraph{Assumptions.}
Throughout this subsection, we assume: (1) The analysis considers joint workspace rotations that transform the robot, task-relevant scene geometry, and the rigidly mounted wrist camera and tactile sensor together. (2) nuisance factors such as background and illumination remain fixed, so the corresponding sensor-local visual and tactile observations remain unchanged before workspace correction. (3) The implemented policy is exactly equivariant only to the selected finite task symmetry group \(G\). In the deployed policy, the symmetry group \(G\) in Eq.~\eqref{eq:equivariance_goal} is instantiated as a sampled finite subgroup of \(\SO(3)\), such as \(I_{60}\) or \(C_8\).

Under these assumptions, VISTA satisfies the workspace-level equivariance objective in Eq.~\eqref{eq:equivariance_goal}:

\begin{proposition}[Workspace equivariance of VISTA]
\label{prop:vista_equivariance}
The implemented VISTA policy satisfies $\pi(g\cdot\OO_t) = g\cdot\pi(\OO_t), \forall g\in G$.
\end{proposition}

\begin{proof}
Consider one observation \(o_t\). Under a global rotation \(g\), the robot end effector, vision camera, tactile sensor, objects, and contact geometry are rotated together. By the sensor-local observation invariance assumption we made earlier, the local visual and tactile features remain unchanged prior to correction. Therefore, the uncorrected fused coefficients satisfy $c_{\ell,t}(g\cdot o_t) = c_{\ell,t}(o_t)$. The end-effector orientation, however, transforms as $R_{g\cdot o_t}=gR_t.$
Using the equivariance correction, we obtain $
    \widehat{c}_{\ell,t}(g\cdot o_t) = D^\ell(R_{g\cdot o_t})c_{\ell,t}(g\cdot o_t) = D^\ell(gR_t)c_{\ell,t}(o_t) = D^\ell(g)D^\ell(R_t)c_{\ell,t}(o_t) = D^\ell(g)\widehat{c}_{\ell,t}(o_t)$.
Thus, the corrected visuotactile observation representation is equivariant. The proprioceptive encoder uses the same scalar and vector transformation rules, so its output is equivariant as well. The resulting representation can therefore be processed by an equivariant diffusion network to output equivariant actions.
\end{proof}

%% file: sections/experiment.tex
\section{Experiments}
\label{sec:experiments}

We evaluate VISTA on simulated and real contact-rich manipulation tasks. The experiments are organized to answer the following questions:
    \textbf{(1)} Does VISTA improve contact-rich imitation learning under a limited demonstration budget?
    \textbf{(2)} Does tactile spherical representation provide benefits beyond an image-only equivariant policy?
    \textbf{(3)} Does tactile-conditioned spherical fusion improve over simpler ways of injecting tactile?
    \textbf{(4)} Does VISTA improve data efficiency as the number of demonstrations changes?
    \textbf{(5)} Does the same design transfer to real-robot contact-rich manipulation?

\subsection{Simulation Experiments}
\label{sec:sim_experiments}

\noindent\textbf{Experimental setup.}\;
We compare with \textbf{ManiFeel}~\cite{luu2025manifeel}, a non-equivariant visuotactile diffusion policy. \textbf{VITAL}~\cite{george2025vital} and \textbf{UniVTAC}~\cite{chen2026univtac} are Action-Chunking Transformer (ACT)~\cite{zhao2023learning}-based policies that fuse pretrained visual and tactile encoders without an explicit workspace-equivariant representation. \textbf{EquiDiff + Tac}~\cite{EquiDiff}, adopted from \cite{luu2025manifeel}, provides the closest equivariant comparison: it uses a \(C_8\)-equivariant visual ResNet, treats tactile features as trivial \(C_8\) representations, and conditions a \(C_8\)-equivariant diffusion U-Net. It therefore has equivariant decoding and tactile input, but no tactile spherical lifting or visual--tactile fusion before harmonic projection. For a fair comparison, all baselines and ablations use the same end-effector-based 10-D absolute action interface as VISTA. All methods share the demonstration source and evaluation protocol. More details are in Appendix~\ref{app:sim_protocol_details}. Moreover, we use a wide-FOV wrist-camera setting with focal length 1.0 so the camera view covers the task workspace.

\begin{figure}[t]
    \centering
    \includegraphics[width=\linewidth]{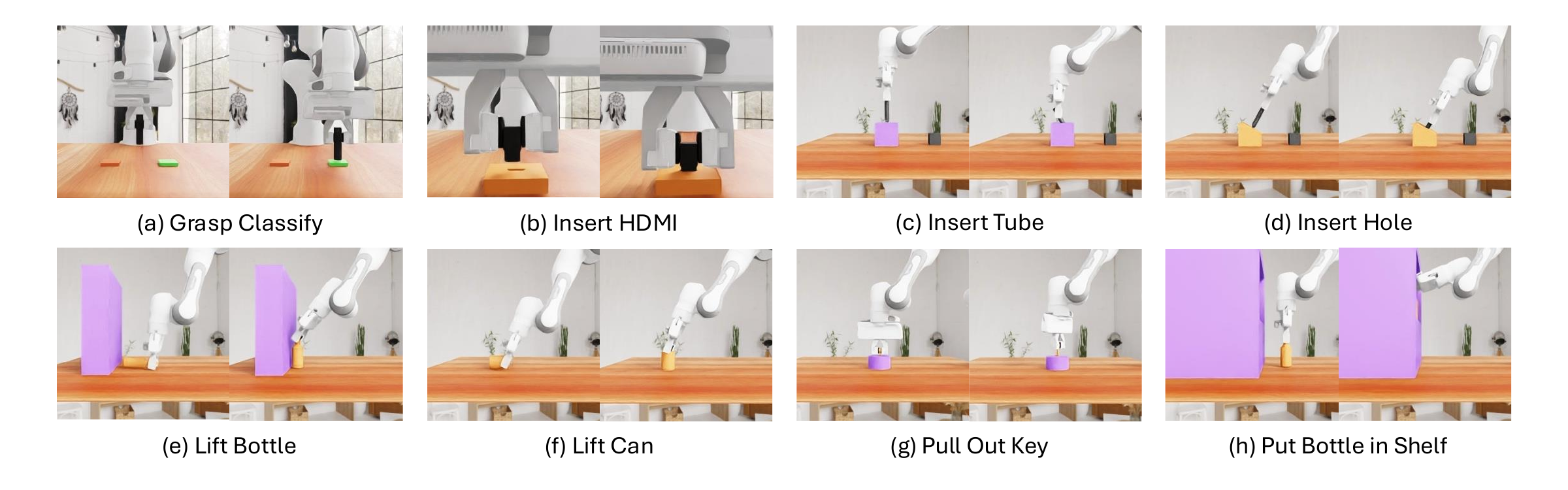}
    \caption{Simulation environments from UniVTAC~\cite{chen2026univtac}. The benchmark includes insertion, extraction, lifting, placement, and contact-based classification tasks.}
    \label{fig:sim_environments}
    \vspace{-3mm}
\end{figure}

\noindent\textbf{Main results.}\;
Table~\ref{tab:simulation_results_main_revised} reports the main simulation results. VISTA achieves the highest mean success rate, 80.9\%. The comparison first shows that fusing vision and tactile sensing alone is insufficient for data-efficient spatial generalization. ManiFeel, VITAL, and UniVTAC all use visuotactile observations, but they largely treat visual and tactile signals as generic input features and do not inject geometric constraints. With only 50 demonstrations, this leads to weak performance on geometry-sensitive tasks such as insert HDMI, insert hole, lift bottle, and lift can. In contrast, VISTA improves the mean success rate over ManiFeel, VITAL, and UniVTAC by 32.1, 37.1, and 59.6 percentage points, respectively. This supports our central motivation: encoding the workspace rotation structure into the policy class reduces the spatial coverage that demonstrations must provide. 

EquiDiff + Tac, on the other hand, uses an equivariant diffusion decoder with tactile input but no spherical lifting or spherical fusion~\cite{luu2025manifeel,EquiDiff}. VISTA still improves over this strongest baseline by 15.9 percentage points. This indicates that equivariant action generation with tactile input is insufficient on its own. Together with the controlled ablations in Sec.~\ref{sec:ablation_study}, this result provides empirical evidence that the proposed tactile spherical lifting and fusion are more effective than treating tactile features as generic conditioning.

\begin{table}[t]
\centering
\caption{Success rates in simulation. All methods are trained with 50 demonstrations per task and evaluated over 40 rollouts.}
\label{tab:simulation_results_main_revised}
\begingroup
\small
\setlength{\tabcolsep}{3.4pt}
\renewcommand{\arraystretch}{1.18}
\begin{tabular}{@{}lccccccccc@{}}
\toprule
\textbf{Method} &
\makecell{\textbf{Grasp}\\\textbf{Classify}} &
\makecell{\textbf{Insert}\\\textbf{HDMI}} &
\makecell{\textbf{Insert}\\\textbf{Hole}} &
\makecell{\textbf{Insert}\\\textbf{Tube}} &
\makecell{\textbf{Lift}\\\textbf{Bottle}} &
\makecell{\textbf{Lift}\\\textbf{Can}} &
\makecell{\textbf{Pull Out}\\\textbf{Key}} &
\makecell{\textbf{Put Bottle}\\\textbf{in Shelf}} &
\textbf{Mean} \\
\midrule
ManiFeel & \textbf{100.0} & 12.5 & 32.5 & 45.0 & 55.0 & 10.0 & 95.0 & 40.0 & 48.8 \\
VITAL & \textbf{100.0} & 7.5 & 27.5 & 55.0 & 20.0 & 17.5 & 65.0 & 57.5 & 43.8 \\
UniVTAC & 60.0 & 5.0 & 17.5 & 32.5 & 0.0 & 0.0 & 37.5 & 17.5 & 21.3 \\
EquiDiff + Tac & 95.0 & 10.0 & 42.5 & 92.5 & 57.5 & 37.5 & \textbf{97.5} & \textbf{87.5} & 65.0 \\
VISTA & \textbf{100.0} & \textbf{40.0} & \textbf{55.0} & \textbf{95.0} & \textbf{80.0} & \textbf{92.5} & \textbf{97.5} & \textbf{87.5} & \textbf{80.9} \\
\bottomrule
\end{tabular}
\endgroup
\vspace{-3mm}
\end{table}

\subsection{Ablation Study}
\label{sec:ablation_study}

We conduct ablations under the same simulation protocol to validate key components of VISTA. Results are summarized in Figure~\ref{fig:sim_ablation_efficiency_summary}, with per-task success rates and variant definitions in Appendix~\ref{app:detailed_sim_results} and Appendix~\ref{app:ablation_variants}. For clarity, the ablations differ only in how tactile information enters the shared ISP-style backbone. \textbf{ISP} is the image-only version of VISTA. \textbf{VISTA-global-\(l0\)} adds a global tactile vector only to the \(\ell=0\) coefficients. \textbf{Early Fusion} concatenates 2D visual and tactile maps before spherical lifting. \textbf{ISP+Concat-Tac} separately lifts both modalities and concatenates spherical tokens without cross-token interaction. \textbf{VISTA} uses visual tokens as queries and tactile tokens as keys/values before harmonic projection. The visual backbone, equivariance correction, decoder, action interface, and finite group are otherwise fixed.

\textbf{(a) Contribution of the tactile pathway.} Adding VISTA's tactile spherical lifting and fusion raises the eight-task mean from 68.1\% to 80.9\%, a 12.8-point improvement under the same protocol. This shows that the visual spherical backbone alone is insufficient and that contact observations provide information complementary to wrist RGB. \textbf{(b) Tactile-conditioned spherical fusion.} Holding the visual spherical backbone fixed, VISTA-global-\(l0\), Early Fusion, and ISP+Concat-Tac achieve 69.4\%, 72.8\%, and 74.1\%, respectively, compared with 80.9\% for VISTA. Thus, generic tactile conditioning improves over ISP by only 1.3--6.0 points, whereas VISTA gains 12.8 points and exceeds the strongest alternative fusion scheme by 6.8 points. This suggests that allowing each visual spherical direction to query tactile tokens before harmonic projection provides a more effective interface for the downstream equivariant backbone. \textbf{(c) Demonstration budget.} With 50 demonstrations, VISTA reaches 80.9\%, already 5.3 points above EquiDiff + Tac trained with 100 demonstrations (75.6\%). Increasing VISTA to 100 demonstrations further raises the mean to 84.7\% (+3.8 points). The strong low-data result and continued improvement with more data indicate improved data efficiency without preventing further scaling.

\begin{figure}[H]
    \centering
    \includegraphics[width=0.90\linewidth]{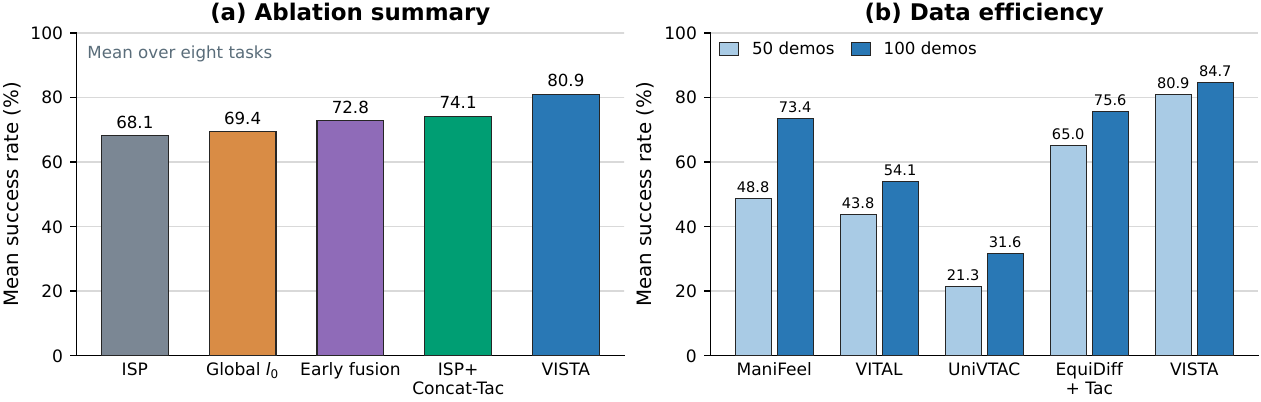}
    \caption{\small Summary of simulation ablations. Left: eight-task-mean fusion ablation over ISP, VISTA-global-\(l0\), Early fusion, ISP+Concat-Tac, and VISTA (definitions in Appendix~\ref{app:ablation_variants}). Right: demo-budget ablation over all eight simulation tasks with 50 or 100 demonstrations. Detailed per-task results are provided in Appendix~\ref{app:detailed_sim_results}.}
    \vspace{-6mm}
    \label{fig:sim_ablation_efficiency_summary}
\end{figure}

\begin{wrapfigure}{r}{0.48\linewidth}
    \centering
    \includegraphics[width=\linewidth]{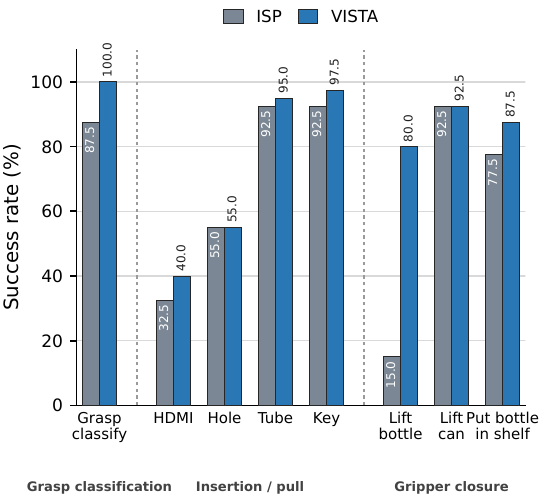}
    \vspace{-4mm}
    \caption{\small \textbf{Task-level ISP--VISTA comparison.} Success rates on all eight simulation tasks, grouped into grasp classification, insertion/pull, and tasks requiring gripper-closure control.}
    \label{fig:isp_vista_task_comparison}
    \vspace{-4mm}
\end{wrapfigure}

\noindent\textbf{(d) Tactile gains for different tasks.}\; \textbf{First,} on grasp classification, ISP already achieves \(87.5\%\), as wrist RGB retains visual cues for different textures. VISTA reaches \(100\%\), showing that the fused policy distinguishes all evaluated classes from the visuotactile observations. \textbf{Second,} across insertion and pull-out tasks, the gains are relatively small: \(7.5\), \(0.0\), \(2.5\), and \(5.0\) points on HDMI, hole, tube, and key. This is because tactile information becomes more informative after meaningful contact. Moreover, unlike EquiTac \cite{zhu2025equitac}, VISTA does not explicitly model contact intrinsic geometry which could be used to correct pre-contact visual misalignment. \textbf{Third,} among three tasks that require gripper-closure control, Lift Bottle improves from \(15.0\%\) to \(80.0\%\), while Put Bottle in Shelf improves from \(77.5\%\) to \(87.5\%\). These results show that VISTA effectively leverages tactile cues under pronounced or sustained contact deformation, where tactile feedback becomes particularly informative for task execution.

\vspace{-2mm}
\subsection{Real-Robot Experiments}
\label{sec:real_experiments}

We evaluate VISTA on three real-robot contact-rich tasks as shown in Figure \ref{fig:real_environments}: insert USB, insert pencil, and erase. These tasks involve tight insertion or sustained surface contact, requiring tactile feedback for contact detection, pose correction, alignment, and stable interaction. The wrist RGB camera and tactile sensor are both mounted on the end effector. We use the same observation and action interface as in simulation, including wrist RGB, tactile RGB, proprioception, and 10-D action chunks. Each method is trained from real demonstrations and evaluated over 10 trials per task. We compare against EquiDiff + Tac, the strongest non-VISTA equivariant baseline in simulation, to test whether adding tactile input to an equivariant diffusion policy is sufficient on hardware. Further protocol details and offline model inference latency are provided in Appendix~\ref{app:real_protocol_details} and ~\ref{app:inference_latency}.

\begin{table}[H]
\centering
\caption{Success rates on real-robot contact-rich manipulation tasks. Mean reports the average number of successful trials out of 10 across tasks.}
\label{tab:real_world_results_revised}
\begingroup
\small
\setlength{\tabcolsep}{6pt}
\renewcommand{\arraystretch}{1.2}
\begin{tabular}{@{}lcccc@{}}
\toprule
\textbf{Method} &
\textbf{Erase} &
\textbf{Insert USB} &
\textbf{Insert pencil} &
\textbf{Mean} \\
\midrule
EquiDiff + Tac & 1/10 & 3/10 & 0/10 & 1.33/10 \\
VISTA & 9/10 & 5/10 & 8/10 & 7.33/10 \\
\bottomrule
\end{tabular}
\endgroup
\vspace{-2mm}
\end{table}

\vspace{-2mm}

Table~\ref{tab:real_world_results_revised} shows that VISTA achieves 7.33/10 mean success across tasks, compared with 1.33/10 for EquiDiff + Tac. The gains are most pronounced on insert pencil and erase, where tactile feedback helps correct small pose errors and maintain stable contact. Although insert USB remains challenging due to tight tolerances, VISTA still outperforms the baseline. These results align with simulation, indicating that contact-rich manipulation benefits from VISTA's spherical representation.

\begin{figure}[H]
    \centering
    \includegraphics[width=\linewidth]{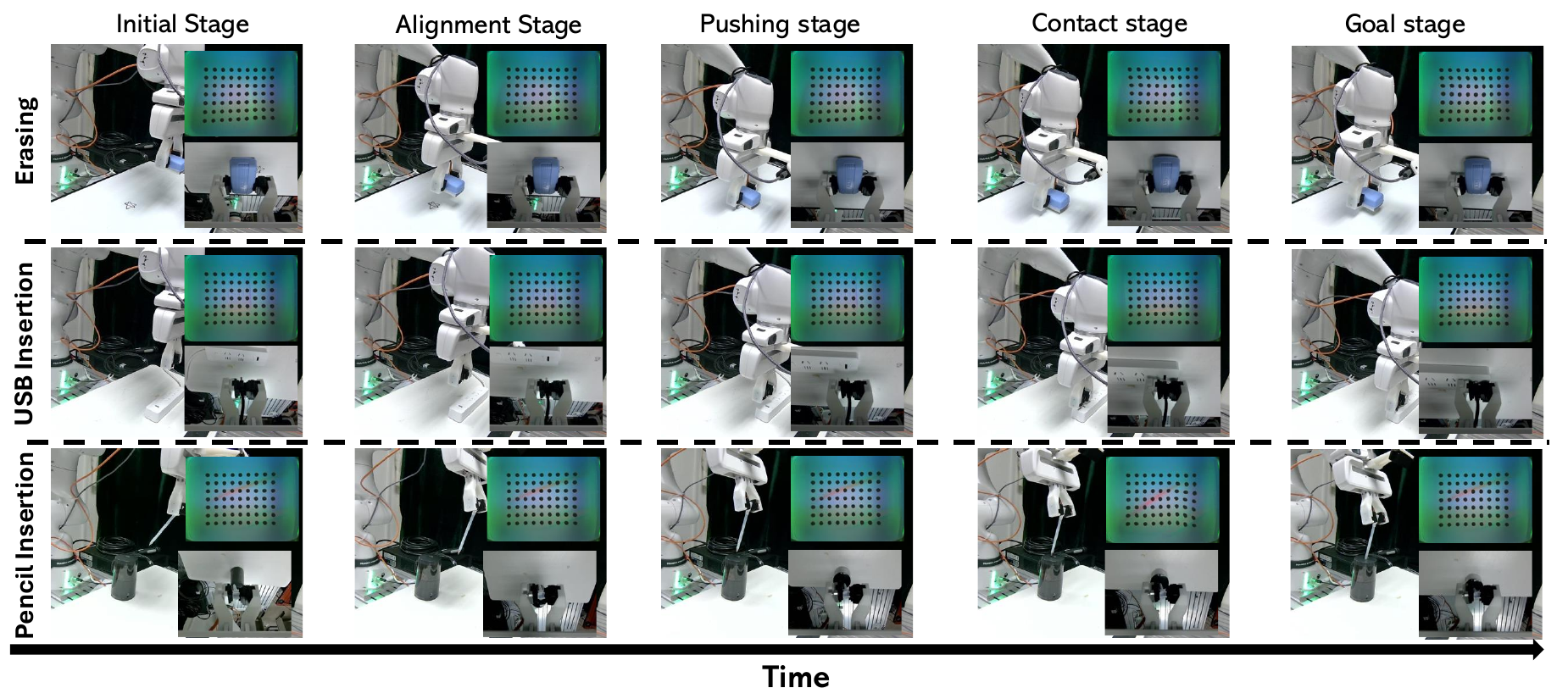}
    \caption{Real-robot task rollouts from VISTA, covering tight insertion and sustained surface-contact manipulation.}
    \label{fig:real_environments}
    \vspace{-4mm}
\end{figure}

%% file: sections/conclusion.tex
\section{Limitations}
\label{sec:limitations}
First, the equivariance guarantee is workspace-level and exact only for the selected finite group; it does not provide intrinsic tactile equivariance to arbitrary in-hand object rotations. Second, the equivariance analysis assumes rigidly mounted sensors and joint robot--object--sensor rotations that leave the local visual and tactile observations unchanged before workspace correction. Background or illumination changes, gravity-dependent contact, in-gripper slip, calibration errors, external cameras, and changing extrinsics can violate this assumption. Explicitly modeling local tactile symmetries and calibration uncertainty remains future work.

Finally, despite sub-\(100\)\,ms model-side inference, the deployed policy executes six insertion actions or eight erasing actions at \(10\)\,Hz before replanning. New tactile observations thus affect decisions only every \(0.6\)--\(0.8\)\,s; the system is not a high-rate tactile closed loop and may miss transient slip or contact changes. We have not studied the trade-off of shorter chunks. Adaptive or event-triggered replanning and fast tactile residual control are promising extensions.

\section{Conclusion}
\label{sec:conclusion}

We presented VISTA, a workspace-level equivariant visuotactile diffusion policy for data-efficient contact-rich manipulation. VISTA projects wrist RGB and tactile observations into spherical tokens, lets each visual direction query tactile features before harmonic projection, and aligns the fused representation with the end-effector orientation before equivariant diffusion decoding. Across eight simulation tasks, VISTA reaches an 80.9\% mean success rate with only 50 demonstrations per task, improving over the image-only spherical policy by 12.8 points and over EquiDiff + Tac by 15.9 points. With 100 demonstrations, it reaches 84.7\%, and on three real-robot tasks it achieves 7.33/10 successful trials compared with 1.33/10 for the strongest equivariant baseline. These results support the value of combining tactile spherical lifting and direction-dependent fusion with workspace-level equivariance. Future work could study broader object distributions, faster tactile feedback, multiple tactile streams, and automatic selection of task symmetries.

%% file: sections/acknowledgments.tex
\section*{Acknowledgments}

The work described in this paper was supported in part by a grant from the NSFC/RGC Joint Research Scheme sponsored by the Research Grants Council of the Hong Kong Special Administrative Region, China and the National Natural Science Foundation of China (Project No.~N\_CUHK410/23), in part by grants from the Research Grants Council of the Hong Kong Special Administrative Region, China (Project Nos.~14208424, C4004-24Y, and T45-401/22-N), and in part by a grant from the National Natural Science Foundation of China (Project No.~62322318).

%% file: sections/appendix.tex
\section{Appendix}
\label{appendix}

This appendix provides additional implementation and experimental details for VISTA. We first describe the tensor interfaces and architectural components used by the observation encoder and equivariant diffusion policy. We then give details on spherical lifting, tactile-conditioned fusion, equivariance correction, diffusion implementation, ablation variants, the simulation and real-robot protocols, and inference latency. Finally, we report full task-level simulation results and data-efficiency numbers.

\begin{itemize}
    \item \textbf{VISTA Implementation Details} in Appendix~\ref{app:vista_details} describe the observation and action tensors, the main tensor chain, and the coefficient interfaces used by the policy.

    \item \textbf{Spherical Lifting Details} in Appendix~\ref{app:spherical_lifting} specify the sphere grid, visual and tactile projection operators, and projection from fused sphere tokens to spherical harmonic coefficients.

    \item \textbf{Proof of Fusion Properties} in Appendix~\ref{app:fusion_proof} provides the full proof that the tactile-conditioned spherical fusion module is equivariant to visual-token permutations.

    \item \textbf{Equivariance Correction Details} in Appendix~\ref{app:equivariance_correction_details} explain how local sensor-frame spherical coefficients are aligned to the workspace frame using the end-effector orientation.

    \item \textbf{Equivariant Diffusion Details} in Appendix~\ref{app:diffusion_details} describe the action representation, finite-group sampling procedure, weight sharing over group samples, and denoising objective.

    \item \textbf{Ablation Variants} in Appendix~\ref{app:ablation_variants} define ISP, ISP+Concat-Tac, Early Fusion, VISTA-global-\(l0\), and VISTA, including where tactile information is injected.

    \item \textbf{Simulation Protocol Details} in Appendix~\ref{app:sim_protocol_details} summarize task randomization, finite-group choices, and the main comparison baselines.

    \item \textbf{Real-Robot Protocol Details} in Appendix~\ref{app:real_protocol_details} describe the hardware setup, training data, rollout frequency, evaluation perturbations, and success criteria.

    \item \textbf{Inference Latency} in Appendix~\ref{app:inference_latency} reports offline observation-to-chunk wall-clock time on the deployment GPU, replanning windows, and a two-sensor compute-scaling measurement.

    \item \textbf{Detailed Simulation Results} in Appendix~\ref{app:detailed_sim_results} report full per-task success rates for the main comparison, backbone ablation, fusion ablation, and data-efficiency study.
\end{itemize}

\subsection{VISTA Implementation Details}
\label{app:vista_details}

This appendix provides implementation details for the observation encoder, spherical lifting, fusion variants, and equivariant diffusion module used in VISTA.

\paragraph{Observation and action tensors.}
The observation dictionary contains a wrist RGB image, a tactile RGB image, an end-effector pose, and a gripper state. For an observation horizon \(H_o\), the main tensors are
\[
    I\in\RR^{B\times H_o\times 3\times H\times W},
    \qquad
    \tau\in\RR^{B\times H_o\times 3\times H\times W},
\]
together with end-effector position, end-effector quaternion, and gripper state. In our experiments, \(H_o=2\), the action prediction horizon is \(H_p=16\), and the policy executes the first \(6\) or \(8\) predicted actions before replanning.

Before normalization, we subtract a fixed workspace center \(c_{\mathrm{ws}}=(0,0,0.8)\) from end-effector positions and translational action components. This improves numerical conditioning and does not affect the equivariance argument, since it is a fixed choice of coordinate origin for the dataset.

The observation time axis is flattened into the batch axis before image encoding:
\[
    [B,H_o,\cdots]\rightarrow [BH_o,\cdots].
\]
Let \(N=BH_o\). The wrist image and tactile image are then encoded into feature maps of spatial size \(7\times 7\):
\[
    F^I\in\RR^{N\times C_I\times 7\times 7},
    \qquad
    F^T\in\RR^{N\times C_T\times 7\times 7}.
\]
If the image encoder produces a different spatial size, adaptive average pooling is used to obtain the \(7\times 7\) interface.

\paragraph{Main tensor chain.}
The following table summarizes the main tensor shapes used by the observation encoder and diffusion policy.

\begin{table}[h]
\centering
\caption{Main tensor shapes in VISTA. Here \(N=BH_o\), \(P\) is the number of sampled sphere points, and \(N_G\) is the number of finite-group samples.}
\label{tab:appendix_tensor_shapes}
\begingroup
\small
\setlength{\tabcolsep}{5pt}
\renewcommand{\arraystretch}{1.15}
\begin{tabular}{@{}ll@{}}
\toprule
Stage & Shape \\
\midrule
Wrist RGB input & \([B,H_o,3,H,W]\) \\
Tactile RGB input & \([B,H_o,3,H,W]\) \\
Flattened image inputs & \([N,3,H,W]\) \\
Visual feature map & \([N,C_I,7,7]\) \\
Tactile feature map & \([N,C_T,7,7]\) \\
Visual sphere tokens & \([N,P,C_I]\) \\
Tactile sphere tokens & \([N,P,C_T]\) \\
Fused sphere tokens & \([N,P,C]\) \\
Spherical harmonic coefficients & \([N,C,(l_{\max}+1)^2]\) \\
$SO3$ coefficients after spherical stack & \([N,128,\sum_{\ell=0}^{l_{\max}}(2\ell+1)^2]\) \\
Observation output & \([B,H_o,128,\sum_{\ell=0}^{l_{\max}}(2\ell+1)^2]\) \\
Finite-group observation samples & \([B,H_o,128,N_G]\) \\
Finite-group action samples & \([B,H_p,64,N_G]\) \\
Diffusion U-Net input & \([BN_G,H_p,64]\) \\
Predicted action noise & \([B,H_p,10]\) \\
\bottomrule
\end{tabular}
\endgroup
\end{table}

In the current implementation, \(l_{\max}=6\). Therefore, \((l_{\max}+1)^2=49\) for spherical coefficients and \(\sum_{\ell=0}^{6}(2\ell+1)^2=455\) for \(\SO(3)\) coefficients.

\subsection{Spherical Projection Details}
\label{app:spherical_lifting}

\paragraph{Sphere grid.}
VISTA samples a set of sphere points \(\{s_p\}_{p=1}^{P}\subset\Sph\) using a HEALPix grid restricted to the forward-facing hemisphere. Each point has spherical harmonic values \(Y_{p,k}\) for degrees up to \(l_{\max}\), and a quadrature weight \(\omega_p\). In implementation, we use uniform weights normalized by \(\sqrt{P}\).

Although spherical harmonics are defined on the full sphere, VISTA uses the hemisphere as a fixed sensor-domain sampling window rather than as an exact full-sphere quadrature rule. The resulting coefficients should therefore be interpreted as learned, band-limited sensor features expressed in a spherical harmonic basis, not as unbiased estimates of a complete spherical signal. This does not change the equivariance correction used by VISTA: the hemisphere is fixed in the local end-effector frame, and global workspace rotations act through the end-effector orientation before the equivariant diffusion decoder.

\paragraph{Visual projection.}
The visual branch follows the image-to-sphere lifting interface used in prior spherical visual policies~\cite{i2s}. Let \(g_{hw}\in\RR^2\) denote the normalized coordinate of a feature-map location and let \(u_p\in\RR^2\) denote the orthographic projection of sphere point \(s_p\) into the feature-map plane. The visual projection weights are initialized by a Gaussian kernel,
\begin{equation}
    \widetilde{W}^{I}_{hwp}
    =
    \exp
    \left(
    -
    \frac{\|g_{hw}-u_p\|_2^2}{2\sigma^2}
    \right),
    \qquad
    W^{I}_{hwp}
    =
    \frac{\widetilde{W}^{I}_{hwp}}
    {\sum_{h',w'}\widetilde{W}^{I}_{h'w'p}}.
    \label{eq:appendix_visual_projection}
\end{equation}
The visual sphere tokens are computed as
\begin{equation}
    v_{n,p,c}
    =
    \operatorname{ReLU}
    \left(
    \sum_{h,w}
    F^I_{n,c,h,w}W^{I}_{hwp}
    \right).
    \label{eq:appendix_visual_tokens}
\end{equation}

\paragraph{Tactile projection.}
The tactile branch uses a separate learnable projection. We intentionally do not use the visual Gaussian initialization for tactile features. The reason is that the visual projection assumes that 2D feature-map coordinates correspond to viewing directions, while tactile RGB pixels correspond to local deformation patterns on the sensor surface. Their relation to workspace directions depends on contact mechanics, object geometry, and the local contact configuration. A visual-style Gaussian initialization would therefore impose an incorrect spatial prior on tactile features.

Instead, VISTA learns tactile lifting weights from scratch:
\begin{equation}
    W^{T}_{hwp}
    =
    \frac{
    \exp(L^{T}_{hwp})
    }{
    \sum_{h',w'}
    \exp(L^{T}_{h'w'p})
    },
    \label{eq:appendix_tactile_projection}
\end{equation}
where \(L^{T}\) is a learned parameter tensor. The tactile sphere tokens are
\begin{equation}
    \tau_{n,p,c}
    =
    \operatorname{ReLU}
    \left(
    \sum_{h,w}
    F^T_{n,c,h,w}W^{T}_{hwp}
    \right).
    \label{eq:appendix_tactile_tokens}
\end{equation}
When multiple tactile streams are used, VISTA applies this projection to each stream, adds a learned tactile-side embedding, and concatenates the resulting token sets.
\paragraph{Projection to spherical harmonics.}
Given fused sphere tokens \(f_{n,p,c}\), VISTA computes spherical harmonic coefficients by
\begin{equation}
    C_{n,c,k}
    =
    \sum_{p=1}^{P}
    f_{n,p,c}Y_{p,k}\omega_p.
    \label{eq:appendix_harmonic_coeff}
\end{equation}
These coefficients are then corrected by the end-effector orientation and passed to the spherical coefficient stack described in the main text.

\subsection{Proof of Fusion Properties}
\label{app:fusion_proof}

This section gives the full proof of Proposition~\ref{prop:index_equivariance_fusion}. For compactness, write the visual token matrix as \(V\in\RR^{P\times d_v}\) and the tactile token matrix as \(Z\in\RR^{P_t\times d_t}\). The attention module computes
\begin{equation}
    Q=VW_q,
    \qquad
    K=ZW_k,
    \qquad
    U=ZW_u,
\end{equation}
and
\begin{equation}
    \Lambda
    =
    \operatorname{softmax}
    \left(
    \frac{QK^\top}{\sqrt{d_h}}
    \right),
    \qquad
    M=\Lambda U,
\end{equation}
where the softmax is applied row-wise over tactile tokens. The fused output is
\begin{equation}
    \mathcal{F}(V,Z)
    =
    \operatorname{LN}
    \left(
    R(V)+\psi([V,M])
    \right),
    \label{eq:appendix_fusion_matrix}
\end{equation}
where \(R\), \(\psi\), and layer normalization are applied token-wise.

\paragraph{Visual-token permutation equivariance.}
Let \(\Pi\in\RR^{P\times P}\) be a permutation matrix over visual sphere tokens. If \(V'=\Pi V\), then
\[
    Q' = V'W_q = \Pi Q.
\]
The attention logits become
\[
    \frac{Q'K^\top}{\sqrt{d_h}}
    =
    \Pi
    \frac{QK^\top}{\sqrt{d_h}}.
\]
Since the softmax is row-wise, permuting the rows before softmax permutes the rows after softmax:
\[
    \Lambda'=\Pi\Lambda.
\]
Thus
\[
    M'=\Lambda'U=\Pi\Lambda U=\Pi M.
\]
Finally, because \(R\), \(\psi\), and layer normalization are token-wise operations,
\[
    \mathcal{F}(\Pi V,Z)
    =
    \Pi\mathcal{F}(V,Z).
\]
This proves visual-index equivariance.

\subsection{Equivariance Correction Details}
\label{app:equivariance_correction_details}

The fused spherical coefficients are initially expressed in a local sensor-rig frame. Let \(R_t\in\SO(3)\) be the end-effector orientation in the workspace frame. For each harmonic degree \(\ell\), VISTA applies
\[
    \widehat{c}_{\ell,t}
    =
    D^\ell(R_t)c_{\ell,t}.
\]
Under a global workspace rotation \(g\), the local visual and tactile measurements are unchanged, while the end-effector orientation becomes \(gR_t\). Therefore,
\begin{align}
    \widehat{c}_{\ell,t}(g\cdot o_t)
    &=
    D^\ell(gR_t)c_{\ell,t}(o_t)
    \\
    &=
    D^\ell(g)D^\ell(R_t)c_{\ell,t}(o_t)
    \\
    &=
    D^\ell(g)\widehat{c}_{\ell,t}(o_t).
\end{align}
This is the coefficient-level equivariance used by the observation encoder.

The tactile sensor may have a fixed extrinsic transform relative to the end effector. This does not change the argument: a fixed extrinsic only changes the local sensor-rig coordinate convention and can be absorbed into the learned tactile lifting and fusion modules. The global transformation still acts by left multiplication on the end-effector orientation.

\subsection{Equivariant Diffusion Details}
\label{app:diffusion_details}

\paragraph{Action representation.}
Each action contains translation, two rotation axes, and a scalar gripper command. Translation and rotation axes are represented as vector features, while the gripper command is represented as a scalar. This gives a natural equivariant feature structure: the vector components transform under rotations, and the gripper component remains invariant.

\paragraph{Finite-group sampling.}
The observation encoder outputs coefficients of shape
\[
    [B,H_o,128,455],
\]
and the action encoder outputs coefficients of shape
\[
    [B,H_p,64,455].
\]
These coefficients are evaluated on a finite rotation group. The \(\SO(3)\) variants use an icosahedral group \(I_{60}\) with \(N_G=60\) samples, while the \(\SO(2)\) variants use a cyclic group \(C_8\) with \(N_G=8\) samples. After sampling, the tensors have shapes
\[
    [B,H_o,128,N_G],
    \qquad
    [B,H_p,64,N_G].
\]
The group axis is then flattened into the batch axis:
\[
    [B,H_p,64,N_G]\rightarrow [BN_G,H_p,64],
\]
and the same temporal 1D U-Net is applied to every group element.

The group axis is treated as additional samples only after the steerable features have been evaluated on the selected finite group. The temporal U-Net uses shared weights and group-independent normalization layers such as GroupNorm or LayerNorm in the diffusion path; we avoid BatchNorm over the flattened \(BN_G\) axis because batch statistics that mix group samples can introduce group-dependent numerical effects and weaken equivariance.

\paragraph{Why weight sharing gives equivariance.}
For an element \(g\) in the selected finite group, acting on the input permutes the finite-group samples. Since the U-Net is shared across all group elements, applying the U-Net before or after this permutation gives the same result up to the same permutation. The final equivariant decoder maps the permuted group signal back to transformed action vectors and an invariant gripper scalar.

\paragraph{Training objective.}
Training follows the standard diffusion-policy denoising objective. A clean action chunk \(\mathbf{U}\) is corrupted by Gaussian noise at diffusion timestep \(k\). The denoiser predicts the added noise:
\[
    \LL
    =
    \mathbb{E}_{\mathbf{U},k,\eps}
    \left[
    \left\|
    \Gamma_\theta(\mathbf{U}^k,E_\theta(\OO),k)
    -
    \eps
    \right\|_2^2
    \right].
\]
In our implementation, no action dimensions are visible through the conditioning mask during denoising; the policy generates the entire action chunk conditioned on the visuotactile observation history.

\subsection{Ablation Variants}
\label{app:ablation_variants}

The ablation family isolates how tactile information is injected into the ISP-style spherical backbone. Unless noted, all variants share the same visual encoder, visual spherical lifting, harmonic projection, end-effector equivariance correction, finite-group diffusion decoder, demonstration budget, 10D action interface, and task-dependent group in Table~\ref{tab:app_task_group_protocol}. Tactile-conditioned variants use the left tactile stream. Table~\ref{tab:ablation_variant_defs} summarizes the only intended differences.

\begin{table}[htbp]
\centering
\caption{Ablation variants. Fusion is the only intended architectural difference; the visual spherical backbone and decoder are held fixed.}
\label{tab:ablation_variant_defs}
\begingroup
\small
\setlength{\tabcolsep}{3.2pt}
\renewcommand{\arraystretch}{1.15}
\begin{tabular}{@{}llll@{}}
\toprule
\textbf{Variant} & \textbf{Tactile representation} & \textbf{Fusion stage} & \textbf{Fusion operator} \\
\midrule
ISP & none & -- & -- \\
VISTA-global-\(l0\) & invariant global vector & \(\ell=0\) harmonic coefficients & addition \\
Early Fusion & 2D tactile feature map & before spherical lifting & \(1\times 1\) conv on concatenated maps \\
ISP+Concat-Tac & spherical tokens & on the sphere, before harmonics & channel-wise concat + linear \\
VISTA & spherical tokens & on the sphere, before harmonics & cross-attention (visual \(Q\), tactile \(K/V\)) \\
\bottomrule
\end{tabular}
\endgroup
\end{table}

\paragraph{ISP.}
Image-only Image-to-Sphere Policy~\cite{hu20263d}. It lifts wrist RGB to a steerable spherical field, applies harmonic projection and end-effector correction, and decodes actions with the same finite-group diffusion policy. It receives no tactile input and therefore measures the contribution of the visual spherical backbone alone.

\paragraph{VISTA-global-\(l0\).}
The visual branch follows ISP. Tactile RGB is encoded into a single invariant vector and added only to the \(\ell=0\) spherical-coefficient slice after harmonic projection. Higher-order visual harmonics are not modulated by touch. This tests tactile input as global contact context rather than a direction-dependent cue.

\paragraph{Early Fusion.}
Visual and tactile 2D feature maps are concatenated in the image plane and fused by a \(1\times 1\) convolution, group normalization, and ReLU. The fused map is then lifted to the sphere with the visual projection and passed through the same harmonic correction. There is no separate learned tactile sphere. This tests whether 2D feature-map fusion before lifting is sufficient.

\paragraph{ISP+Concat-Tac.}
Visual and tactile maps are lifted separately: visual tokens use the Gaussian viewing-direction projection, and tactile tokens use the learned tactile projection in Appendix~\ref{app:spherical_lifting}. At each sphere index the two tokens are concatenated channel-wise and linearly projected, then passed through harmonic projection and correction. This tests spherical tactile conditioning without cross-attention. Unlike Early Fusion, both modalities occupy the sphere before fusion; unlike VISTA, fusion does not let each visual direction query the tactile tokens.

\paragraph{VISTA.}
The full method. After the same separate visual and tactile spherical liftings as ISP+Concat-Tac, a permutation-equivariant cross-attention block uses visual tokens as queries and tactile tokens as keys/values, then applies residual token-wise fusion before harmonic projection. This is the only variant that injects tactile evidence into each visual spherical direction through learned attention.

\paragraph{$SO(3)$ and $SO(2)$ instantiations.}
Each variant above can be paired with \(I_{60}\) or \(C_8\). Reported simulation numbers use the same task-dependent choice as VISTA in Table~\ref{tab:app_task_group_protocol}.

\subsection{Simulation Protocol Details}
\label{app:sim_protocol_details}

\paragraph{Checkpoint selection and evaluation.}
All simulation methods are trained with a single seed. Demonstrations are split \(8{:}2\) into training and validation sets, and the checkpoint with the lowest validation diffusion loss is retained. Reported success rates then use 40 rollouts with seeds 1000--1039; these rollouts are not used to choose the checkpoint.

\paragraph{Task randomization and group choice.}
Table~\ref{tab:app_task_group_protocol} summarizes the main randomization and finite-group instantiation used for each simulation task. We use \(I_{60}\) when the task requires full 3D orientation handling, and \(C_8\) when the dominant symmetry is rotation around gravity or the task is primarily planar/axial.

\begin{table}[htbp]
\centering
\caption{Simulation task randomization and VISTA group instantiation. The ranges summarize the data-generation protocol used by the benchmark.}
\label{tab:app_task_group_protocol}
\begingroup
\small
\setlength{\tabcolsep}{4pt}
\renewcommand{\arraystretch}{1.12}
\begin{tabular}{@{}lll@{}}
\toprule
Task & VISTA group & Main randomized factors \\
\midrule
Grasp classify & \(C_8\) & object class, grasp height, gripper closing pose \\
Insert HDMI & \(I_{60}\) & slot \(x/y\), in-hand pitch/yaw, staging pose noise \\
Insert hole & \(I_{60}\) & slot \(x/y\), slot yaw, target pitch, grasp height \\
Insert tube & \(I_{60}\) & slot \(x/y\), grasp height, staging pose noise \\
Lift bottle & \(C_8\) & bottle \(x/y\), bottle yaw, grasp pitch \\
Lift can & \(C_8\) & can size, can \(x/y\), grasp pitch \\
Pull out key & \(C_{8}\) & slot \(x/y\), slot yaw, key yaw, grasp depth, over-rotation \\
Put bottle in shelf & \(C_8\) & shelf \(x\), bottle \(y\), grasp height/pitch, lift height \\
\bottomrule
\end{tabular}
\endgroup
\end{table}

\paragraph{Main comparison baselines.}
All methods predict the same 10D absolute end-effector action and use the same demonstration source. Sharing the action interface does not imply identical sensors: ManiFeel, UniVTAC, and EquiDiff + Tac use wrist RGB with both left and right tactile RGB in simulation, VITAL uses wrist RGB with tactile RGB, and VISTA uses wrist RGB with the left tactile stream. On the real robot, EquiDiff + Tac and VISTA both use the single left GelSight Mini.

\paragraph{ManiFeel.}
A non-equivariant visuotactile diffusion policy~\cite{luu2025manifeel}. Wrist RGB and each tactile RGB image are encoded by independent ResNet-18 branches, concatenated with proprioception, and used as global conditioning for a 1D diffusion U-Net. It does not use spherical lifting or a rotation-equivariant decoder.

\paragraph{VITAL.}
An ACT-style CVAE transformer~\cite{george2025vital} with CLIP-pretrained vision and GelSight tactile encoders. Wrist RGB, tactile RGB, and proprioception are fused in the transformer rather than on a sphere.

\paragraph{UniVTAC.}
The official UniVTAC ACT baseline~\cite{chen2026univtac}, using a pretrained tactile reconstruction encoder on left and right tactile streams together with wrist RGB and proprioception. Like VITAL, it is a transformer imitation policy without workspace equivariance.

\paragraph{EquiDiff + Tac.}
A tactile-conditioned EquiDiff policy~\cite{EquiDiff}. Wrist RGB is encoded by a \(C_8\)-equivariant ResNet. Each tactile RGB image is encoded by a ResNet-18 whose features are treated as trivial \(C_8\) representations. These features are concatenated with proprioception, mapped by a \(C_8\)-equivariant linear layer, and used to condition a \(C_8\)-equivariant diffusion U-Net. EquiDiff + Tac therefore has an equivariant decoder and tactile input, but no spherical tactile lifting, no visual--tactile spherical fusion, and no \(I_{60}\) instantiation. It is trained and evaluated with the same data, action representation, horizons, and protocol as VISTA. The real-robot comparison uses this same architecture with the single left GelSight stream.

\subsection{Real-Robot Protocol Details}
\label{app:real_protocol_details}

The real-robot experiments use a Franka arm with an end-effector-mounted Orbbec wrist RGB camera and a left GelSight Mini tactile sensor. Each task uses 100 real demonstrations for training; all demonstrations are real-only and no simulation checkpoint is transferred. We train with a single seed, split the demonstrations \(8{:}2\) into training and validation sets, save checkpoints every 30 epochs, and select the snapshot with the lowest validation diffusion loss. The 10 reported trials per task are fully held-out and are not used for checkpoint selection. During rollout, the policy observes two history frames, predicts action chunks, and executes actions at 10 Hz; the Franka-side controller runs at 300 Hz. The wrist RGB stream runs at 30 Hz and the GelSight Mini stream runs at 10 Hz.

For insert USB and insert pencil, the evaluation initial end-effector pose is uniformly perturbed by \(\pm 1\) cm along each Cartesian axis, while the socket/sharpener is placed randomly within the visible workspace and the in-hand object is given a rotational offset. In erase, the whiteboard is placed on the table and the marker stroke is drawn randomly within the wrist-camera-visible area. A trial is counted successful if the USB is inserted into the socket, the pencil is inserted and activates the sharpener, or the target marker stroke is removed by wiping.

\subsection{Inference Latency}
\label{app:inference_latency}

We report \emph{offline model-side} inference latency: the wall-clock time from a batched observation tensor already on GPU to a denoised action chunk. This is not a full sensor-to-actuation delay; image capture, host-side preprocessing, and robot communication are excluded. Timing uses a single RTX~3090, batch size 1, and DDIM with 8 denoising steps~\cite{song2020denoising}, matching the real-robot sampler. Numbers in Table~\ref{tab:inference_latency} are averages over repeated forward passes of the trained one-tactile models.

\begin{table}[htbp]
\centering
\caption{Offline observation-to-chunk inference latency on one RTX~3090 with DDIM-8. EquiDiff + Tac uses its \(C_8\) decoder.}
\label{tab:inference_latency}
\begingroup
\small
\setlength{\tabcolsep}{8pt}
\renewcommand{\arraystretch}{1.15}
\begin{tabular}{@{}lcc@{}}
\toprule
\textbf{Method} & \textbf{Group} & \textbf{Latency (ms)} \\
\midrule
ISP & \(C_8\) & 61.25 \\
ISP & \(I_{60}\) & 65.66 \\
VISTA & \(C_8\) & 63.74 \\
VISTA & \(I_{60}\) & 71.54 \\
EquiDiff + Tac & \(C_8\) & 79.88 \\
\bottomrule
\end{tabular}
\endgroup
\end{table}

Actions are played at \(10\)\,Hz. On the real robot the policy executes 8 predicted actions before replanning for erase (\(0.8\)\,s) and 6 actions for insertion (\(0.6\)\,s). All tabulated latencies are far below these windows, so a complete diffusion pass finishes before the next replan. This does not make the controller a \(10\)\,Hz tactile closed loop: newly observed tactile frames are used only at the next chunk. Scaling the decoder group from \(C_8\) to \(I_{60}\) increases VISTA latency by \(7.8\)\,ms at nearly identical parameter count.

Task experiments use one GelSight stream. Separately, we timed a randomly initialized two-stream VISTA-\(I_{60}\) that encodes each tactile image, lifts both to the shared sphere, and concatenates tokens with a tactile-side embedding. Latency rose from \(70.18\)\,ms (one stream) to \(73.13\)\,ms (two streams) in that isolated measurement; it is a compute-scaling check, not a trained two-sensor policy, and is not mixed with Table~\ref{tab:inference_latency}.

\subsection{Detailed Simulation Results}
\label{app:detailed_sim_results}

This section reports the full task-level simulation results in the main paper and Figure ~\ref{fig:sim_ablation_efficiency_summary}.

\begin{table}[htbp]
\centering
\caption{\textbf{Main simulation comparison.} Success rates in simulation. All methods are trained with 50 demonstrations per task and evaluated over 40 rollouts. Missing entries are marked as -- and are excluded from the mean.}
\label{tab:simulation_results_main_revised_app}
\begingroup
\small
\setlength{\tabcolsep}{3.4pt}
\renewcommand{\arraystretch}{1.18}
\begin{tabular}{@{}lccccccccc@{}}
\toprule
\textbf{Method} &
\makecell{\textbf{Grasp}\\\textbf{Classify}} &
\makecell{\textbf{Insert}\\\textbf{HDMI}} &
\makecell{\textbf{Insert}\\\textbf{Hole}} &
\makecell{\textbf{Insert}\\\textbf{Tube}} &
\makecell{\textbf{Lift}\\\textbf{Bottle}} &
\makecell{\textbf{Lift}\\\textbf{Can}} &
\makecell{\textbf{Pull Out}\\\textbf{Key}} &
\makecell{\textbf{Put Bottle}\\\textbf{in Shelf}} &
\textbf{Mean} \\
\midrule
ManiFeel & \textbf{100.0} & 12.5 & 32.5 & 45.0 & 55.0 & 10.0 & 95.0 & 40.0 & 48.8 \\
VITAL & \textbf{100.0} & 7.5 & 27.5 & 55.0 & 20.0 & 17.5 & 65.0 & 57.5 & 43.8 \\
UniVTAC & 60.0 & 5.0 & 17.5 & 32.5 & 0.0 & 0.0 & 37.5 & 17.5 & 21.3 \\
EquiDiff + Tac & 95.0 & 10.0 & 42.5 & 92.5 & 57.5 & 37.5 & \textbf{97.5} & \textbf{87.5} & 65.0 \\
VISTA & \textbf{100.0} & \textbf{40.0} & \textbf{55.0} & \textbf{95.0} & \textbf{80.0} & \textbf{92.5} & \textbf{97.5} & \textbf{87.5} & \textbf{80.9} \\
\bottomrule
\end{tabular}
\endgroup
\end{table}

\begin{table}[htbp]
\centering
\caption{\textbf{ISP versus VISTA.} Backbone ablation in simulation. ISP is the image-only spherical policy baseline; VISTA adds tactile spherical lifting and tactile-conditioned spherical fusion.}
\label{tab:isp_vista_ablation_revised_app}
\begingroup
\small
\setlength{\tabcolsep}{4pt}
\renewcommand{\arraystretch}{1.15}
\begin{tabular}{@{}lccccccccc@{}}
\toprule
\textbf{Method} &
\makecell{\textbf{Grasp}\\\textbf{Classify}} &
\makecell{\textbf{Insert}\\\textbf{HDMI}} &
\makecell{\textbf{Insert}\\\textbf{Hole}} &
\makecell{\textbf{Insert}\\\textbf{Tube}} &
\makecell{\textbf{Lift}\\\textbf{Bottle}} &
\makecell{\textbf{Lift}\\\textbf{Can}} &
\makecell{\textbf{Pull Out}\\\textbf{Key}} &
\makecell{\textbf{Put Bottle}\\\textbf{in Shelf}} &
\textbf{Mean} \\
\midrule
ISP & 87.5 & 32.5 & \textbf{55.0} & 92.5 & 15.0 & 92.5 & 92.5 & 77.5 & 68.1 \\
VISTA & \textbf{100.0} & \textbf{40.0} & \textbf{55.0} & \textbf{95.0} & \textbf{80.0} & \textbf{92.5} & \textbf{97.5} & \textbf{87.5} & \textbf{80.9} \\
\bottomrule
\end{tabular}
\endgroup
\end{table}

\begin{table}[htbp]
\centering
\caption{\textbf{Fusion ablation.} Fusion ablation on all eight simulation tasks. VISTA-global-\(l0\) injects tactile information as invariant context, Early Fusion fuses 2D feature maps before spherical projection, ISP+Concat-Tac concatenates visual and tactile spherical tokens, and VISTA performs tactile-conditioned spherical fusion before harmonic projection.}
\label{tab:fusion_ablation_revised_app}
\begingroup
\small
\setlength{\tabcolsep}{3.4pt}
\renewcommand{\arraystretch}{1.15}
\begin{tabular}{@{}lccccccccc@{}}
\toprule
\textbf{Method} &
\makecell{\textbf{Grasp}\\\textbf{Classify}} &
\makecell{\textbf{Insert}\\\textbf{HDMI}} &
\makecell{\textbf{Insert}\\\textbf{Hole}} &
\makecell{\textbf{Insert}\\\textbf{Tube}} &
\makecell{\textbf{Lift}\\\textbf{Bottle}} &
\makecell{\textbf{Lift}\\\textbf{Can}} &
\makecell{\textbf{Pull Out}\\\textbf{Key}} &
\makecell{\textbf{Put Bottle}\\\textbf{in Shelf}} &
\textbf{Mean} \\
\midrule
ISP & 87.5 & 32.5 & 55.0 & 92.5 & 15.0 & \textbf{92.5} & 92.5 & 77.5 & 68.1 \\
VISTA-global-\(l0\) & \textbf{100.0} & 17.5 & 52.5 & 90.0 & 77.5 & 77.5 & 57.5 & 82.5 & 69.4 \\
Early Fusion & \textbf{100.0} & \textbf{42.5} & 37.5 & 90.0 & 75.0 & 77.5 & 82.5 & 77.5 & 72.8 \\
ISP+Concat-Tac & \textbf{100.0} & 37.5 & \textbf{67.5} & 85.0 & 67.5 & \textbf{92.5} & 77.5 & 65.0 & 74.1 \\
VISTA & \textbf{100.0} & 40.0 & 55.0 & \textbf{95.0} & \textbf{80.0} & \textbf{92.5} & \textbf{97.5} & \textbf{87.5} & \textbf{80.9} \\
\bottomrule
\end{tabular}
\endgroup
\end{table}

\begin{table}[H]
\centering
\caption{\textbf{Data efficiency.} Detailed data-efficiency comparison in simulation. Missing entries are marked as -- and are excluded from the mean. VISTA is reported as the task-appropriate instantiation selected from task geometry before training.}
\label{tab:app_data_efficiency_app}
\begingroup
\setlength{\tabcolsep}{2.8pt}
\renewcommand{\arraystretch}{1.14}
\begin{tabular}{@{}lcccccccccc@{}}
\toprule
\textbf{Method} &
\textbf{\#Demo} &
\makecell{\textbf{Grasp}\\\textbf{Classify}} &
\makecell{\textbf{Insert}\\\textbf{HDMI}} &
\makecell{\textbf{Insert}\\\textbf{Hole}} &
\makecell{\textbf{Insert}\\\textbf{Tube}} &
\makecell{\textbf{Lift}\\\textbf{Bottle}} &
\makecell{\textbf{Lift}\\\textbf{Can}} &
\makecell{\textbf{Pull Out}\\\textbf{Key}} &
\makecell{\textbf{Put Bottle}\\\textbf{in Shelf}} &
\textbf{Mean} \\
\midrule
ManiFeel & 50 & \textbf{100.0} & 12.5 & 32.5 & 45.0 & 55.0 & 10.0 & 95.0 & 40.0 & 48.8 \\
ManiFeel & 100 & \textbf{100.0} & 25.0 & 60.0 & 92.5 & 70.0 & 77.5 & 95.0 & 67.5 & 73.4 \\
VITAL & 50 & \textbf{100.0} & 7.5 & 27.5 & 55.0 & 20.0 & 17.5 & 65.0 & 57.5 & 43.8 \\
VITAL & 100 & \textbf{100.0} & 10.0 & 20.0 & 72.5 & 50.0 & 45.0 & 72.5 & 62.5 & 54.1 \\
UniVTAC & 50 & 60.0 & 5.0 & 17.5 & 32.5 & 0.0 & 0.0 & 37.5 & 17.5 & 21.3 \\
UniVTAC & 100 & 60.0 & 0.0 & 15.0 & 42.5 & 0.0 & 10.0 & 65.0 & 60.0 & 31.6 \\
EquiDiff + Tac & 50 & 95.0 & 10.0 & 42.5 & 92.5 & 57.5 & 37.5 & 97.5 & \textbf{87.5} & 65.0 \\
EquiDiff + Tac & 100 & 97.5 & 27.5 & 72.5 & \textbf{95.0} & 75.0 & 77.5 & \textbf{100.0} & 60.0 & 75.6 \\
VISTA & 50 & \textbf{100.0} & 40.0 & 55.0 & \textbf{95.0} & \textbf{80.0} & 92.5 & 97.5 & \textbf{87.5} & 80.9 \\
VISTA & 100 & \textbf{100.0} & \textbf{42.5} & \textbf{80.0} & 92.5 & \textbf{80.0} & \textbf{100.0} & 97.5 & 85.0 & \textbf{84.7} \\
\bottomrule
\end{tabular}%
\endgroup
\end{table}